\documentclass{svproc}
\usepackage{url}

\usepackage{graphicx,import} 
\usepackage{transparent}
\usepackage{booktabs}
\usepackage{mathrsfs}
\usepackage[inline]{enumitem}
\usepackage{circuitikz}
\usetikzlibrary{shapes,shadows,arrows,automata}
\usepackage{todonotes}
\usepackage{algorithm}
\usepackage{algpseudocode}
\algnewcommand{\IfThenElse}[3]{
  \State \algorithmicif\ #1\ \algorithmicthen\ #2\ \algorithmicelse\ #3}
\algnewcommand\algorithmicforeach{\textbf{for each}}
\algdef{S}[FOR]{ForEach}[1]{\algorithmicforeach\ #1\ \algorithmicdo}

\usepackage{subfig}
\usepackage{amsmath}
\usepackage{amsfonts}
\usepackage{paralist}
\usepackage{siunitx}

\begin{document}
\graphicspath{{figures/}}
\mainmatter              
\title{Collective transport via sequential caging}
%
%
\author{Vivek Shankar Vardharajan\inst{1} \and Karthik Soma \inst{2} \and Giovanni Beltrame\inst{1}}
\authorrunning{V Vardharajan} 
%
%
\institute{Polytechnique Montreal, Montreal QC H3T 1J4, Canada,\\
\email{vivek-shankar.varadharajan@polymtl.ca},
\and
National Institute of Technology, Tiruchirapalli, India}
\maketitle              

\begin{abstract} 
  We propose a decentralized algorithm to collaboratively transport
  arbitrarily shaped objects using a swarm of robots. Our approach
  starts with a task allocation phase that sequentially distributes
  locations around the object to be transported starting from a seed
  robot that makes first contact with the object. Our approach does
  not require previous knowledge of the shape of the object to ensure
  caging. To push the object to a goal location, we estimate the
  robots required to apply force on the object based on the angular difference between the target and the object. During transport, the robots follow a sequence of
  intermediate goal locations specifying the required pose of the
  object at that location. We evaluate our approach in a physics-based
  simulator with up to 100 robots, using three generic paths. 
  Experiments using a group of KheperaIV robots demonstrate the
  effectiveness of our approach in a real setting.
   
\end{abstract}

\section{Introduction}
Several insect species exhibit an incredible level of coordination to
lift and carry heavy objects to their nest, whether for building
materials or food. Paratrechina longicornis can collectively transport
heavy food from a source location to its nest purely through local
interaction with neighboring ants~\cite{Gelblum2015}. These ants are
capable of carrying an object of arbitrary shape and ten times heavier
than their bodies by collaborating in an effective manner. Designing
approaches to realize collaborative transport using a group of robots
can find its application in warehouse
management~\cite{rosenfeld2016human} and collaborative construction of
structures~\cite{Petersen2019}.
\begin{figure}
	\centering
	\def\svgwidth{0.7\linewidth}
\begingroup%
  \makeatletter%
  \providecommand\color[2][]{%
    \errmessage{(Inkscape) Color is used for the text in Inkscape, but the package 'color.sty' is not loaded}%
    \renewcommand\color[2][]{}%
  }%
  \providecommand\transparent[1]{%
    \errmessage{(Inkscape) Transparency is used (non-zero) for the text in Inkscape, but the package 'transparent.sty' is not loaded}%
    \renewcommand\transparent[1]{}%
  }%
  \providecommand\rotatebox[2]{#2}%
  \newcommand*\fsize{\dimexpr\f@size pt\relax}%
  \newcommand*\lineheight[1]{\fontsize{\fsize}{#1\fsize}\selectfont}%
  \ifx\svgwidth\undefined%
    \setlength{\unitlength}{672.23524757bp}%
    \ifx\svgscale\undefined%
      \relax%
    \else%
      \setlength{\unitlength}{\unitlength * \real{\svgscale}}%
    \fi%
  \else%
    \setlength{\unitlength}{\svgwidth}%
  \fi%
  \global\let\svgwidth\undefined%
  \global\let\svgscale\undefined%
  \makeatother%
  \begin{picture}(1,0.47105494)%
    \lineheight{1}%
    \setlength\tabcolsep{0pt}%
    \put(0,0){\includegraphics[width=\unitlength,page=1]{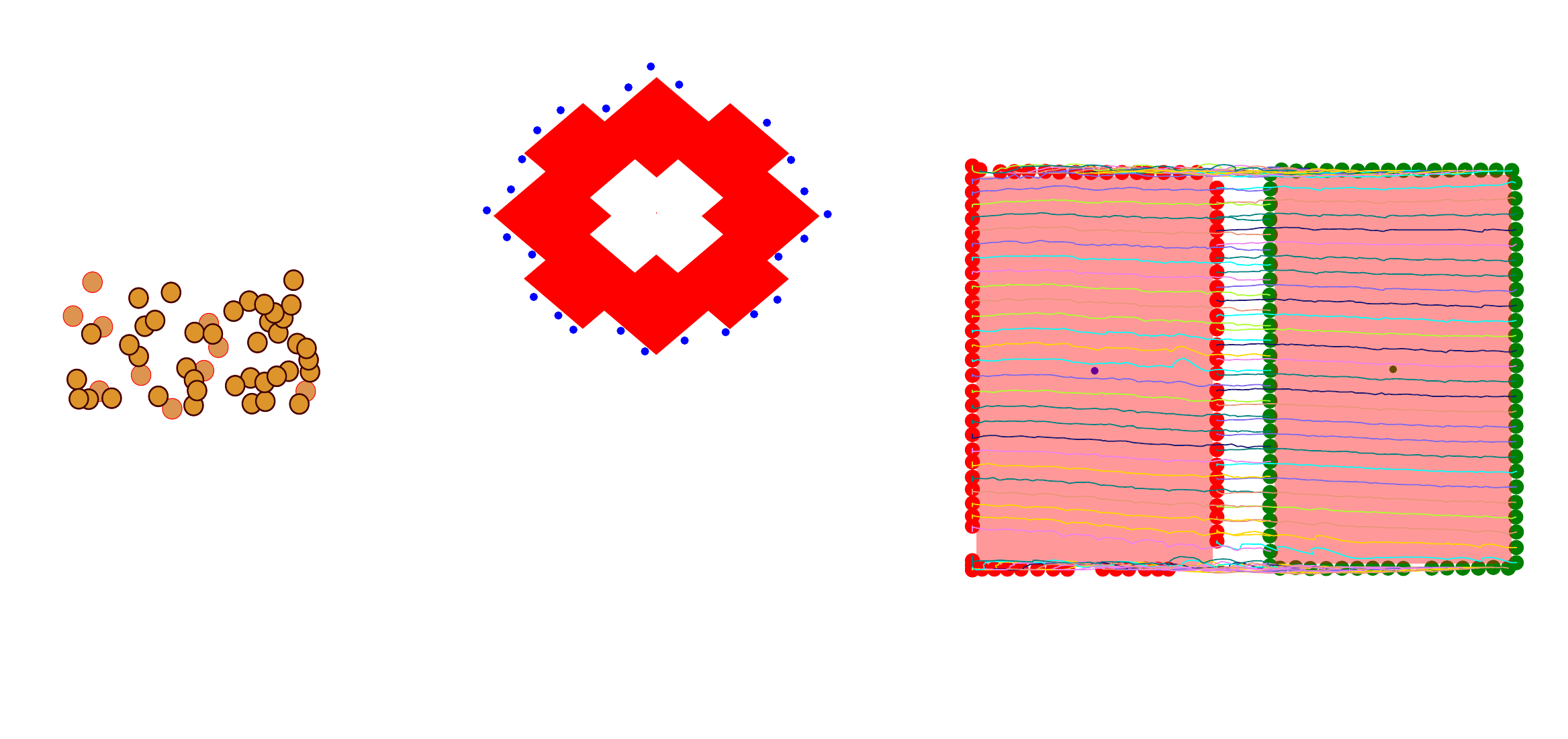}}%
    \put(0.12966044,0.40052161){\color[rgb]{0,0,0}\makebox(0,0)[t]{\lineheight{1.25}\smash{\begin{tabular}[t]{c}Deployment \\Cluster\end{tabular}}}}%
    \put(0.3553476,0.44392995){\color[rgb]{0,0,0}\makebox(0,0)[lt]{\lineheight{1.25}\smash{\begin{tabular}[t]{l}Caging\end{tabular}}}}%
    \put(0.66342642,0.42284735){\color[rgb]{0,0,0}\makebox(0,0)[lt]{\lineheight{1.25}\smash{\begin{tabular}[t]{l}Transportation\end{tabular}}}}%
    \put(0,0){\includegraphics[width=\unitlength,page=2]{Title_fig_copy.pdf}}%
  \end{picture}%
\endgroup%
 
	\caption{Illustration of the process of object transpiration, with the robots starting from a deployment cluster, caging the object and transporting the object.}
	\label{fig:title}
\end{figure}
In this work, we take inspiration from these natural insect species to
design an approach to collaboratively transport heavy objects (i.e. that
cannot be carried by a single robot) of arbitrary shape using
a swarm of robots. The main challenges of this type of
collective transport are:
\begin{inparaenum}
\item the effective placement of robots around the transported object,
\item the effective application of force around the object to avoid tugs-of-war, and
\item adapting to the objects center-of-mass movement and the alignment of forces between
  neighbors.
\end{inparaenum}
Our approach starts with a task allocation phase, where the robots are
sequentially deployed around the object, a process known as
caging~\cite{wan2012}. Completing the task allocation phase, the
robots transport the object towards a target location as shown in fig.~\ref{fig:title}.

Collaborative transport is a well-studied topic: some approaches use
ground
robots~\cite{Petersen2019,Wan2020,fink2007composition,Chen2015}
for cooperative manipulation either with explicit
communication~\cite{Franchi2019} or force based
coordination~\cite{Gabellieri2020}, while others use quadcopters
carrying a
cable-suspended~\cite{cotsakis2019decentralized} or
rigidly attached~\cite{Wang2018} object that is heavier than one
robot's maximum payload. 

Many of the approaches discussed earlier do not explicitly consider
robot's interaction with other robots to continuously maintain cage formation, while including an obstacle
avoidance mechanism within their control framework. This latter is
particularly limiting when the perimeter of the pushed object is
small, as it might prevent the robots to get close to the object to
apply an effective force. In our approach, we design a sequential
placement of robots to avoid this scenario and maintain the
initial formation continuously while moving.
    
Approaches like~\cite{melo2016collaborative} assume that the position
and shape of the object is either continuously known or periodically
updated using a central system. Measuring the position and shape of
the object in a real world scenarios might be difficult and would
limit the use of the transport system to some indoor
applications. Furthermore, some approaches either assume that the
robots are readily placed around the object to be
manipulated~\cite{Rubenstein2013} and design control strategies for
the manipulation of the object. A few approaches provide emphasis on
caging or design a control policy for a specific type of
caging~\cite{Pereira2004}. In this paper, we design a complete system
that allows the robots to start from a deployment cluster and take up
positions around the object to be transported, in a way that is
similar to~\cite{fink2007composition}. Our approach periodically
estimates the centroid of the object using only the relative
positional information shared by neighboring robots, avoiding the need
to have continuous external measurements. Our task allocation allows
heterogeneous robots with varying capabilities (e.g. path planning),
as well as provides a way to addressing robot failures~\cite{Varadharajan2020}. We provide sufficient conditions
for the convergence of our caging, and show that our approach
terminates for convex objects.

\section{Related Work \label{sec:related}}

The concept of caging was first introduced by Rimon and
Blake~\cite{Rimon1999} for a finger gripper. Caging is a concept of
trapping the object to be manipulated by a gripping actuator, in
our work we use a group of robots to act as a single entity to grip
the object using form closure as in~\cite{wan2012}. A simple form of
caging and leader-follower based
strategies~\cite{wang2004control} are employed to push
the object by sensing the resultant forces. The main constraint of
these works is that the robots cannot follow paths with sharp turns.

Pereira et al.~\cite{Pereira2004} introduce a new type of caging
called object closure, in which the object to be transported is
loosely caged until the configuration satisfies certain conditions in
an imaginary \textit{closure configuration space}. Each robot in the
team has to estimate the orientation and position of the object to use
this approach. Wan et al.~\cite{wan2012} propose robust caging to
minimize the number of robots to form closure using translation and
rotation constraints. Wan et al. also extended their work for
polygons, balls and cylinders to be transported on a
slope~\cite{Wan2020}. This approach requires continuous positional
updates from an external system and uses a central system to compute
the minimum number of robots required.

An approach to caging L-shaped objects is proposed
in~\cite{fink2007composition}: the robots switch between
different behavioral states to approach the object and achieve
\textit{potential caging}. This approach requires the robots to know
certain properties of the object beforehand, such as the minimum and
maximum diameter of the caged object. A caging strategy for polygonal
convex objects is proposed in~\cite{dai2016symmetric}, the approach
uses a sliding mode fuzzy controller to traverse predefined paths. A
leader robot coordinates the transportation using the relative position
of all the other robots.

Gerkey et. al.~\cite{gerkey2002pusher} propose a strategy for pushing
an object by assigning ``pushers'' and ``watchers''. Watchers monitor
the position of the object and other robots, while pushers perform the
pushing task. The approach provides fault tolerance to robot failures
and relies heavily on the performance of the watchers. Chen et
al.~\cite{Chen2015} propose an approach in which the robots are placed
around an object that occludes the transportation destination: the
robots that do not see the destination are the ones that push the
object. The intuition behind this placement is that the location that
is most effective for pushing the object is the occluded region,
i.e. the opposite side of the direction of movement. Our strategy is
similar to~\cite{Chen2015}, where we estimate the angular difference between the object and the goal using the proximity sensor, and only the robots that are below a threshold apply a force to the object. Our approach places robots all around the object to adjust and follow a sequence of changing target location, as opposed to only placing robots in the occluded region~\cite{Chen2015}.  

The work in~\cite{wang2006cooperative} combines reinforcement learning
and evolutionary algorithms to coordinate 3 types of agents to learn
to push an object. ``Vision robots'' estimate the positions of the
object and other robots, ``evolutionary learning agents'' generate
plans for the ``worker robots'', and these latter execute the plan to
push the object.  Alkilabi et al.~\cite{alkilabi2017cooperative} use
an evolutionary algorithm to tune a recurrent neural network
controller that allows a group of e-puck robots to collaboratively
transport a heavy object. The robots use an optical flow sensor to
determine and achieve an alignment of forces. The authors demonstrate
that the approach works well for different object sizes and shapes,
however, the proper functioning of the algorithm relies heavily on the
performance of the optical flow sensors.


\section{Methodology \label{sec:methodology}}

Fig.~\ref{fig:State machine} shows the high level state machine used
in our collaborative transport behavior. At initialization, the robots
enter into a sequence of task allocation rounds allowing the robots to
take up positions around the object to be transported with a desired
inter-robot distance. Once caging is complete, the robots that are a
part of the object cage agree on the desired object path and start
pushing and rotating the object. The path is represented as a sequence
of points, and all the robots in the cage use a barrier (i.e. they
wait for consensus) to go through each intermediate point.
\begin{figure}[tbp]
\flushleft
\includegraphics[width=0.99\linewidth]{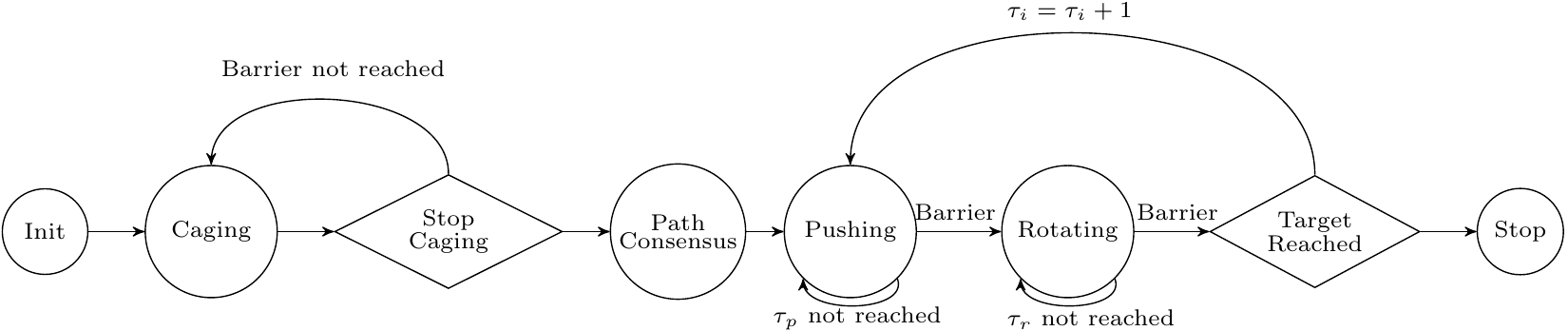}
\caption{High level state diagram}
\label{fig:State machine}
\end{figure}%
\subsection{Problem Formulation \label{subsec:problem}}
Let $C_o(t) \in \mathbb{R}^2$ be the centroid of an arbitrary shaped
object at time $t$, $x_i(t) \in \mathbb{R}^2$ the position of robot
$i$ at time $t$, and $X(t)=\{x_1(t),x_2(t),...,x_n(t)\}$ the set
containing the positions of the robots at time $t$. Let $S_o$ be a
closed, convex set representing the perimeter of the object to be
transported. Given a sequence of target locations
$\mathbb{T}=\{\tau_1,\tau_2,...,\tau_n\}$ the task of the robots is to
take up locations around the perimeter of the object $S_o$ with a
desired inter-robot distance $I_d \in \mathbb{R}$ and drive the
centroid of the object $C_o$ from a known initial state $C_o(0)$ to a
final state $C_o(t_f)$ at some time $t_f$, passing through the target
locations in $\mathbb{T}$.

We assume that the robots can perceive the goal in the environment and
know an estimate of the initial centroid location $C_o(0)$ of the
object to be transported. The mass of the object is assumed to be proportional to the size with a minimum density for the object. We also assume that the line connecting two
subsequent targets $\tau_{x-1}$ and $\tau_x$ does not go through
obstacles and the perimeter of the object $S_o$ to be transported is
greater than $3*I_d$. We consider a point mass model for the robots
($\dot{x}_i(t) = u_i$) and assume that the robots are fully
controllable with $u_i$. The robots are assumed to be equipped with a range and bearing sensor to determine the relative positional information and communicate with neighboring robots within a small fixed communication range $d_C$. In our experiments, we used KheperaIV robots
that are equipped with 8 proximity sensors at equal angles around the
robot, and we assume similar capabilities in general.

\begin{figure}
	\centering
	\def\svgwidth{0.99\linewidth}

	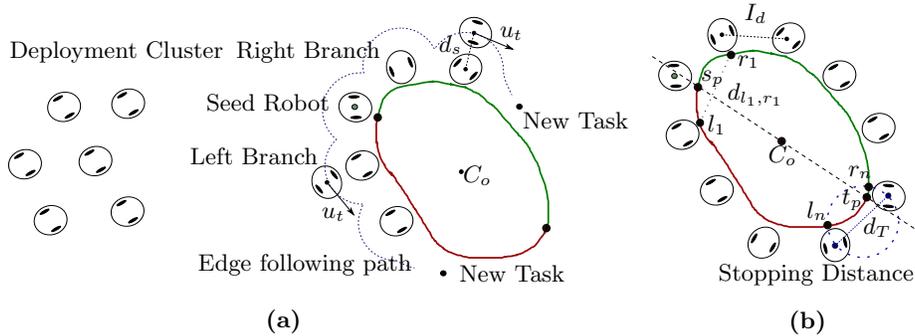 

	\caption{Illustration of the process of task allocation based caging  (a) illustrate the process of edge following to reach the new target (b) the stopping condition to terminate caging.}
	\label{fig:cage}
\end{figure}

\subsection{Task Allocation based Caging}

Consider a group of robots randomly distributed in a cluster and a
known initial position of the object. The goal of the robots is to
deploy to suitable location around the perimeter of the object $C_o$
to guarantee object closure while respecting the required inter-robot
distance $I_d$.

The caging process starts with the allocation of the first task (an
estimation of the centroid of the object) to a seed robot (closest
robot to the object, elected via bidding) as shown in
Figure~\ref{fig:cage} (a). The seed robot moves towards the center of
the object until it detects the object with its on-board proximity
sensors. As the seed robot touches the object, it creates two target
locations (one to its left and one to its right, called
\emph{branches}). The robots bid for these new target locations and
continue the process of spawning the new targets along their branch
until the minimum distance between robots in the two branches is
smaller than $d_T$.

Our task allocation algorithm performs the role of determining the
appropriate target around the object for each robot, the \emph{caging
  targets}. We consider a Single Allocation (SA)
problem~\cite{Choi2009}, where every robot is assigned a single
task. The caging targets are sequentially available to the robots,
i.e. a new target becomes available after a robot has reached its
target. Note that these targets are considered to be approximate
(created by establishing a local coordinate system like
in~\cite{Rubenstein795}), hence they are refined by the robots using
their proximity sensors and the position of the closest robot in the
branch on reaching the assigned target. The approach of sequential caging is particularly appropriate for scenarios where the shape of the object is initially or continuously unknown, the robots sequentially assign robots to the closure and enclose the transported object.

We use a bidding algorithm~\cite{Varadharajan2020} (described in the
supplementary material \footnote{https://mistlab.ca/papers/CollectiveTransport/}): the robots locally compute
bids for a task and recover the lowest bid of the team from a
distributed, shared tuple space~\cite{Pinciroli:vs:2016}. The robots
update the tuple space if the local bid is lower, with conflicts
resolved using the procedure outlined
in~\cite{Varadharajan2020}). After a predetermined allocation time
($T_a$), the lowest bid in the tuple space is declared as the
winner. $T_a$ has to be selected considering the communication
topology and delays to avoid premature allocation as detailed
in~\cite{Varadharajan2020}.

To reach the assigned target, the robots edge-follow the neighbors in
their target branch. The control inputs ($u_i$) are generated using
range and bearing information from neighbors: the robots find their
closest neighbor and create a neighbor vector $x_n$ using the range
and bearing information. The control inputs are then
$u_i= (\perp x_n)+(||x_n||- I_d)*x_n $; the first term makes sure the
robot orbits the neighbor and the second term applies a pulling or
pushing force to keep the robot at distance $I_d$. On reaching the
target (detected using the proximity sensors), the robot measure the
distance to the closest neighbor of the branch and apply a distance
correction to keep the inter-robot caging distance to be $I_d$. The
robot creates an obstacle vector using the proximity values:
$x_o=\frac{\sum_{i \in P} p_i}{|P|}$, with
$P=\{p_0,..,p_7\},p_i\in \mathbb{R}^2$ being the set containing the
individual proximity readings as vectors. The inter-robot distance
correction control inputs are generated using: $u_i=\perp x_o$. When
there are not enough robots to complete the caging, the robots can
adapt by increasing $I_d$ and applying inter-robot distance correction
control until the termination condition is met.

\noindent \textbf{Proposition 1} Consider two sets L and R denoting
the left and right branch respectively, L contains all the attachment
points of the left branch robots to the object, and R contains the
attachment points on the right branch. The caging terminates if
$\exists p \in L, q \in R$ such that $d_{pq} \leq d_{T}$ while
$d_{lr} > d_{T} \quad \forall l \in L-\{ p \} \wedge \forall r \in
R-\{ q \} $.

Referring to fig.~\ref{fig:cage}(b), consider two closed, convex
  sets, $L_o =\{l_{1}, ..., l_{n}\}$ containing all the points on the
  curve $l_{1}l_{n}$ and $R_{o} = \{r_{1}, ..., r_{n}\}$ containing
  all the points on the curve $r_{1}r_{n}$. Set $L_o$ and $R_o$ are ordered and the distance between the two constitutive points satisfy $d(l_n,l{n-1}) > d_T$. 
  

\subsection{Behaviors for pushing and rotating}
Once two robots around the perimeter of the object satisfy the
termination condition and a consensus is reached on the path, the
robots initiate the target following routine. The path is represented
as a sequence of desired object centroid target locations $\mathbb{T}$
and each entry $\tau_i=(\tau_p,\tau_r)$, with
$\tau_p \in \mathbb{R}^2$ and $\tau_r \in [0,2\pi]$. $\tau_p$ is the
local target location and $\tau_r$ is the desired object orientation
along the z-axis (yaw) at that target location. The main intuition
behind having these local targets is to use a geometric path
planner. One of the robot in the swarm with the ability to compute a path to the user defined target $\tau_n$ compute the path and share it as a sequence of states(targets) using virtual stigmergy~\cite{Pinciroli:vs:2016}. The robots sequentially traverse the targets in $\mathbb{T}$,
on reaching a $\tau_p$, the robots rotate the object to $\tau_r$. Each
robot in caging computes $u_{fp}$ as in equ.~\ref{equ:fpush} to exert
a forward force and push the object. Similarly, for rotating the
object the robots apply $u_{fr}$ as in equ.~\ref{equ:frotate}.
\begin{align}
     u_{fp} &= u_t + u_f + u_{cp}, \label{equ:fpush} \\
     u_{fr} &= u_r + u_f + u_{cr} \label{equ:frotate}
\end{align}
where, $u_t$ and $u_r$ are a force to move the object towards the
target by pushing, and a torque to rotate the object to the desired
angle, respectively. $u_f$, as shown in equ.~\ref{equ:formation}, is
the contribution that makes sure the robots stay in the same
formation. $u_{cp}$ (equ.~\ref{equ:center_push}) and $u_{cr}$
(equ.~\ref{equ:center_rot}) are contributions that ensure the robots
stay in contact with the object during pushing and rotation.

\noindent \textbf{Maintaining Formation} The robot formation from the
caging operation tends to get distorted as a result of its application
of pushing force on the object to move it towards the target. The
robots in the cage apply a force to stay in this formation throughout
the transportation task: they store a set
$\textit{N}_{f} =\{(d_i,\theta_{i}) \vert d_i \leq k*I_d, \forall i
\in \textit{N} \}$ that contains the range and bearing measurements of
their neighbors, with $k$ being a design parameter. The control input
$u_f$ to maintain formation is:
\begin{equation}
      \label{equ:formation}      
      u_f =  \sum_{\forall i \in \textit{N}_{f}}^{}{\frac{K_f(d_{i} - d_{cur})}{d_{i}}} \begin{bmatrix}   
           d_{i}\cos{\theta_{i}} - \cos{\frac{\theta_{i} - \theta_{cur}}{\theta_i}}     \\
          d_{i}\sin{\theta_{i}} - \sin{\frac{\theta_{i} - \theta_{cur}}{\theta_i}} 
          \end{bmatrix}
\end{equation}
The first term in the equ.~\ref{equ:formation} is the desired
inter-agent distance correction, while the second term applies the
desired orientation correction. This formulation is inspired from the
commonly used edge potential to preserve a lattice structure among the
robots~\cite{mesbahi2010graph}. We apply this edge potential among adjacent robots in
a cage to preserve the formation and the desired
inter-agent distance.

\noindent \textbf{Maintaining contact with the Object} The robots in
the formation need to determine if they need to apply a force and stay
in contact with the object. During pushing, the robots apply a control
input to stay in contact with the object, determining its
effectiveness in pushing as in equ.~\ref{equ:center_push}.
The effectiveness of a robot's pushing depends on the position of the robot with
respect to the object and the target.

The angle $\theta_p$ (a parameter) determines if the robot is an
effective pusher: if the angle between the object $x_o$ and the target $\tau_{p}$ is greater or equal to $\theta_p$, pushing
is considered effective, and the robots apply an input $u_{cp}$ to
maintain contact. Similarly, the robots apply $u_{cr}$ to maintain
contact during rotations:
\begin{align}
      u_{cp} &=
      \begin{cases} 
      [0,0]^{T}, &   \angle (x_o, \tau_{p}) \geq \theta_p \\
      \frac{K_{cp}x_o}{||x_o||},  &  \angle (x_o, \tau_{p}) < \theta_p     
      \end{cases} \label{equ:center_push} \\
      u_{cr} &=
      \frac{K_{cr} x_o }{|| x_o ||} \label{equ:center_rot}
\end{align} 
where $x_o$ is the proximity vector that determines the current object location
in robots coordinate system, and
$\theta_p$, $K_{cp}$ and $K_{cr}$ are design parameters. $\theta_p$ is a design parameter that defines the effective pushing perimeter around the object, as shown in fig.~\ref{fig:goal_proof}.

\noindent \textbf{Applying forces} The robots have to exert force in
the right direction to move and rotate the object according to the
targets in $\mathbb{T}$. The robots must apply force
in the desired angular window around the perimeter of the object to
avoid tugs-of-war. The control inputs $u_t$ and $u_r$ make sure the
robots exert the force in the right direction:

\begin{align}
 u_{t} &=
 \begin{cases} 
      [0,0]^{T}, & ||\tau_{p_{l}} - x_i|| \leq d_{tol}   \\
      \frac{K_t[\tau_{p_{l}} - x_i]}{||\tau_{p_{l}} - x_i||},  & ||\tau_{p_{l}} - x_i|| > d_{tol}    
   \end{cases} \label{equ:force_push}\\
 u_{r} &= 
 \begin{cases} 
      \begin{bmatrix}   
          0 & -1\\
          1 & 0 
      \end{bmatrix}
      \frac{K_{r} (x_i-C_o)}{||x_i-C_o||}, & \angle (\tau_p,C_o) < \theta_{r}\\   
                  [0,0]^{T},  &   \angle (\tau_p,C_o) \geq \theta_{r}       
 \end{cases} \label{equ:force_rot}
\end{align}
where, $d_{tol}$ is a design parameter that defines the distance
tolerance, $\tau_{p_{l}}$ is the local target computed by the robot
using the centroid position and its position along the perimeter of
the object. On reaching an intermediate target $\tau_i$ the robots
share their approximate position with respect to a common coordinate
system computed as in~\cite{Rubenstein795} in a distributed, shared
tuple space~\cite{Pinciroli:vs:2016} with all the other robots. The
robots retrieve this positional information and compute the centroid
of the object $C_o$, which is then used during the rotation of the
object.

\begin{figure}
	\centering
	\def\svgwidth{0.7\linewidth}
	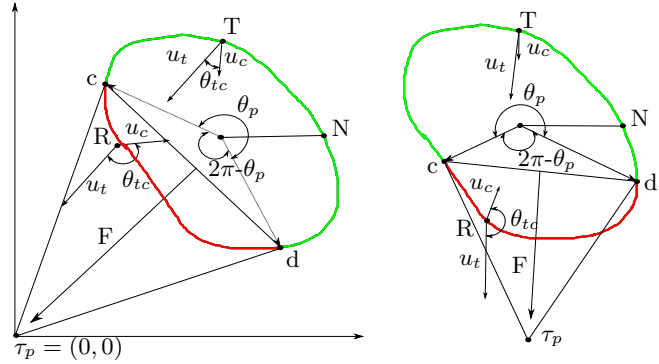 
	\caption{Illustration of the resultant force and the angle of effective robots, the effective pushing positions on the perimeter of the object is shown in green.}
	\label{fig:goal_proof}
\end{figure}

As in~\cite{Chen2015}, we can prove that the object reaches the goal
as $t\to\infty$, with the difference being that the robots exerting
force are not based on the occluded perimeter of the object, but are
instead the robots satisfying $\angle (x_o, \tau_{p}) < \theta_p$.

\begin{theorem}
  The distance between the centroid of the object $C_o$ and the target
  location $\tau_p$ strictly decreases if the velocity of the
  transported object is governed by the translation dynamics equation
  of the object $\dot{v_o}=kF$. For $t\to\infty$, the center of the
  object $C_o$ reaches $\tau_p$, where $\dot{v_o}$ is the derivative
  of the object velocity and $kF$ is the fraction of the force that is
  transfered to the object from the robots.
\end{theorem}

\begin{proof}
Fig.~\ref{fig:goal_proof} shows the resultant F transferred from the
robots to the object and the effective angular window (along the curve
cTd) on the perimeter of the object to exert force. Consider all the
robots along the curve cTd are applying a force using a control input
determined by the unit vector $u_t$. The overall force transferred to
the object is $F = (c_x - c_y) - (d_x - d_y)$, which is the tangent
vector $(d-c)$ rotated by ($\frac{\pi}{2}$)~\cite{Chen2015}.

Consider the squared distance between the target $\tau_p=[0,0]^T$ and
centroid $C_o$ at time t to be $d_g(t) = ||\tau_p - C_o||^2$, taking
the time derivative gives $\dot{d_g}= 2k*C_{o}*F$, substituting F with
the resultant force gives
$\dot{d_g}= k*((C_{o_{y}}c_x - C_{o_{x}}c_y) - (C_{o_{y}}d_x -
C_{o_{x}}d_y))$.  The distance $d_g(t) \geq 0, \forall t > 0$, when
the center of the object lies outside the desired goal $\tau_p$ and
since $C_{o}*F < 0$ is strictly decreasing because of the force
applied by the robots, we get $\lim_{t \to \infty} \dot{d_g}=0$. In
other words, the center $C_{o}$ will eventually reach the goal
$\tau_p$.
\end{proof}

\section{Experiments \label{sec:experiments}}
We performed a set of experiments in a physics-based simulator
(ARGoS3~\cite{Pinciroli:SI2012}) with a KheperaIV robot model under
various conditions to study the performance of our approach. We
implemented our behavioral state machine for the robots using
Buzz~\cite{PinciroliBuzz2016}. We set the number of robots
$N_r \in \{25,50,100\}$ and adapt the size
$S \in \{[2,2],[3.6,6],[7.2,12]\} m$ and mass
$M \in \{5.56,30.024,120.096\} kg$ of a cuboid object according to the
number of robots. The mass of the object is calculated assuming a
constant density hollow material. In another set of experiments, we
used three irregular objects: cloud, box rotation, and clover. We set
the design parameters of the algorithm to the values shown in
Tab.~\ref{tab:exp}. We choose the gain parameters for maintaining
formation($k_f$), contact with object($k_{cp}$) and force application
($k_t$) based on several rounds of trail-and-error simulations. The
tolerance parameters $d_{tol}$ and \textit{Orient. tol.} are chosen
to fit our non-holonomic robots: a large part of the error shown in
fig.~\ref{fig:error} is due to the non-holonomic nature of
our robots. We evaluate the various performance metrics over three
benchmark paths: a straight line, a zigzag, and straight line with two
\ang{90} rotations (straight\_rot in Fig.~\ref{fig:trajectory}). All
the paths consists of 9 waypoints (WPs) and straight\_rot has its
rotations at WP 3 and 6. Each experiment is repeated 30 times with
random initial conditions.
\begin{table}
\caption{Experimental parameters}
\label{tab:exp}
\centering
\begin{tabular}{ccc}
\begin{tabular}{ |c|c| }
 \hline
 \multicolumn{2}{|c|}{Caging} \\
 \hline
 $d_s$ &  0.35m \\
 \hline 
 $I_d$ & 0.45m  \\ 
 \hline
 $d_{tol}$ & 0.05m \\
 \hline
 $d_T$ & $1.85I_d$ \\ 
 \hline
 $K_t$ & 30 \\
 \hline
 Prox. thres. & 0.7\\
 \hline
\end{tabular}
&
\begin{tabular}{ |c|c|  }
 \hline
 \multicolumn{2}{|c|}{Pushing} \\
 \hline
 $\theta_p$ & $115^{\circ}$\\
 \hline
 $K_{cp}$& 40($<115^{\circ}$), 20($\geq115^{\circ}$)\\
 \hline
 $d_{tol}$ &  0.1m\\
 \hline
 $K_{f}$&  40\\
 \hline
 $K_t$ & 60\\
 \hline
 Barrier & 90\%\\
 \hline
\end{tabular}
&
\begin{tabular}{ |c|c|  }
 
 \hline
 \multicolumn{2}{|c|}{Rotating} \\
 \hline
 $K_{cr}$& 450\\
 \hline
 Orient. tol. & $5.72^{\circ}$\\
 \hline
 $K_{f}$& 400\\
 \hline
 $K_r$ & 600\\
 \hline
 Barrier & 90\%\\
 \hline
\end{tabular}
\end{tabular}
\end{table}
\begin{figure}[tbp]
\centering
\includegraphics[width=0.99\linewidth]{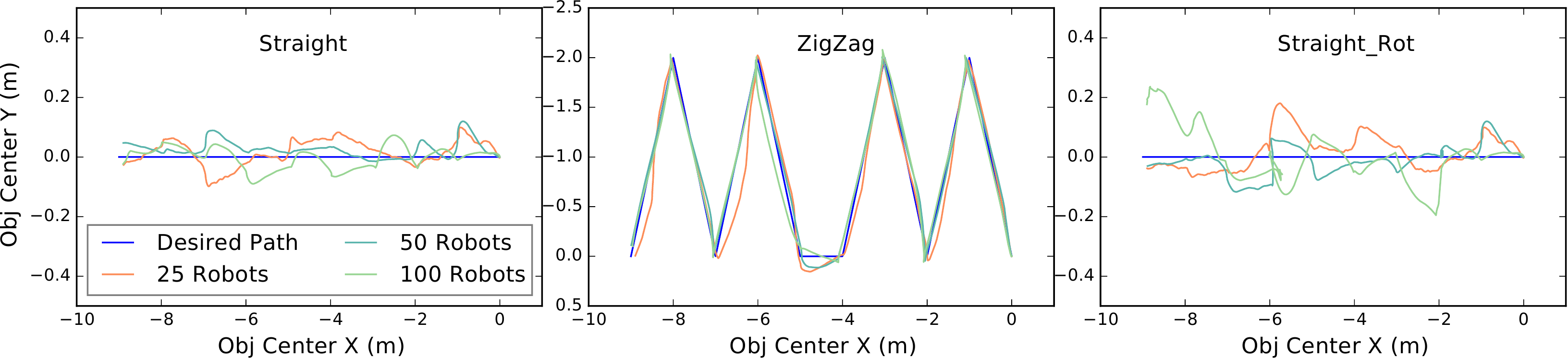}
\caption{Trajectory taken by the centroid of the object vs the desired path in the three benchmarking paths.}
\label{fig:trajectory}
\end{figure}%

\noindent \textbf{Results} We assess the performance of the algorithm
observing the time taken to cage the object and push it along the
benchmark paths, plotted in Fig.~\ref{fig:time}. The time to cage the
object increases with the perimeter of the object: the median times to
cage are \SIlist{247;779;2753}{\second} for \numlist{25;50;100}
robots, respectively transporting objects of size
$\{2,2\}$, $\{3.6,6\}$, $\{7.2,12\}$. The 3 irregular shapes took around
\SI{300}{\second} to cage when using 30 robots. The time taken to push
the object is approximately \SI{100}{\second} for the straight path
(regardless of the object size) and about \SI{160}{\second} for the
zigzag with \numlist{25;50} robots.

When using 100 robots for the straight line and the zigzag, the system
was slightly faster, which could be explained by the higher cumulative
force exerted by the robots. The time taken following a straight path
with rotations increases sub-linearly with the number of robots with
median times being \SIlist{135;155;200}{\second} for
\numlist{25;50;100} robots, respectively.
\begin{figure}[tbp]
\centering
\includegraphics[width=0.99\linewidth]{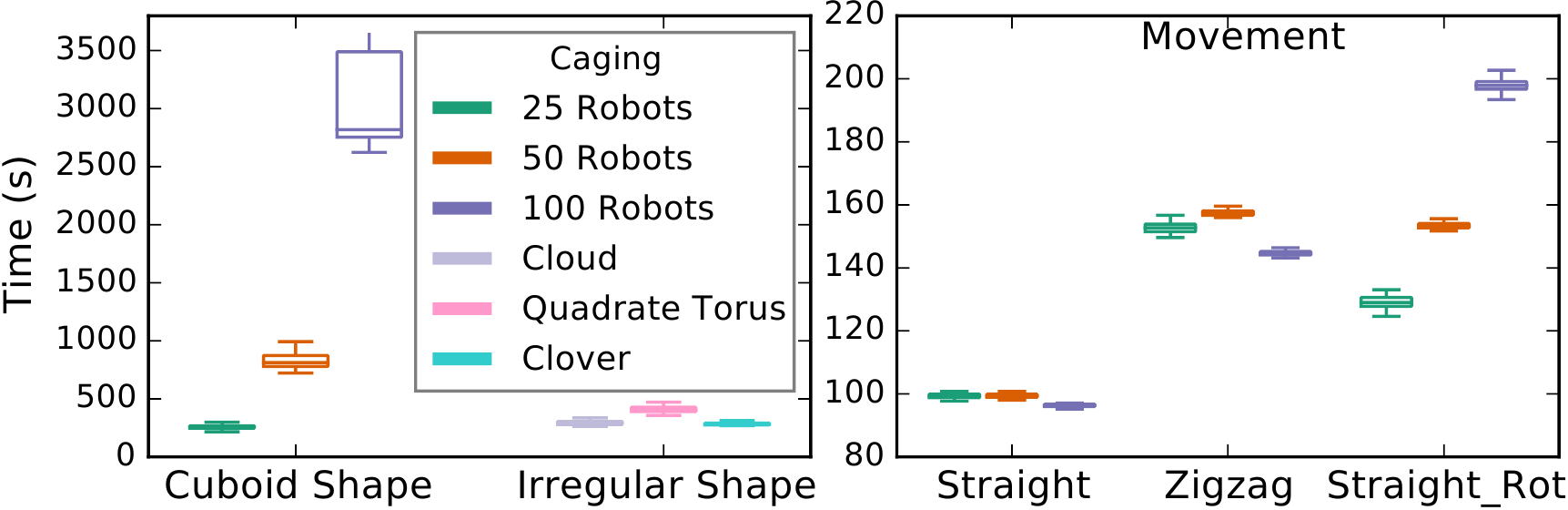}
\caption{Time to complete caging (left) and push an object along 3 paths (right).}
\label{fig:time}
\end{figure}%
\begin{figure}[tbp]
\centering
\includegraphics[width=0.99\linewidth]{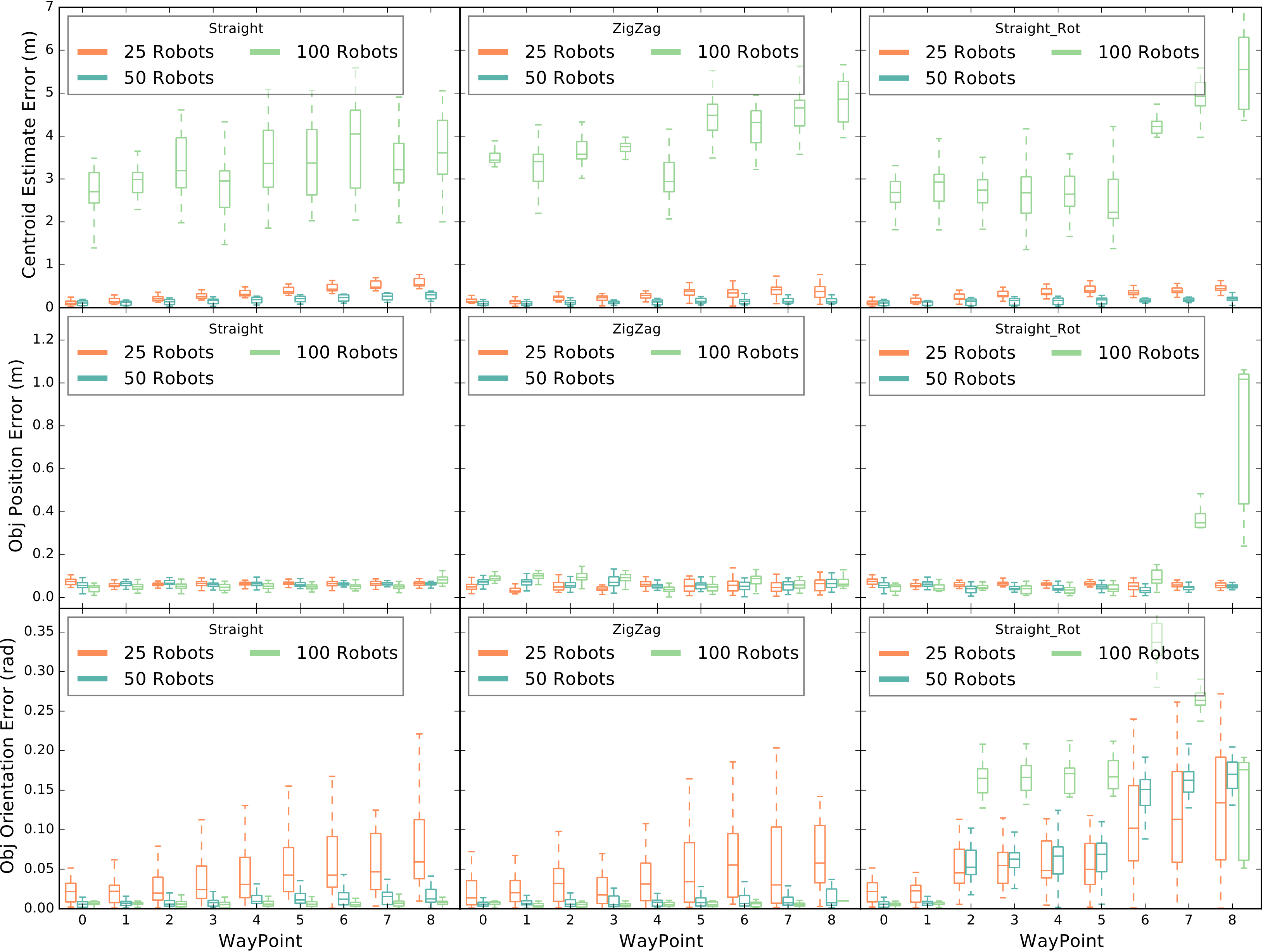}
\caption{From top to bottom: the first plot shows the average centroid
  estimation error, the second shows the object centroid position error, and
  the third shows the object orientation error; from left to right, the
  figure shows the results for our 3 benchmark paths. }
\label{fig:error}
\end{figure}%
\begin{figure}[tbp]
\centering
\includegraphics[width=0.99\linewidth]{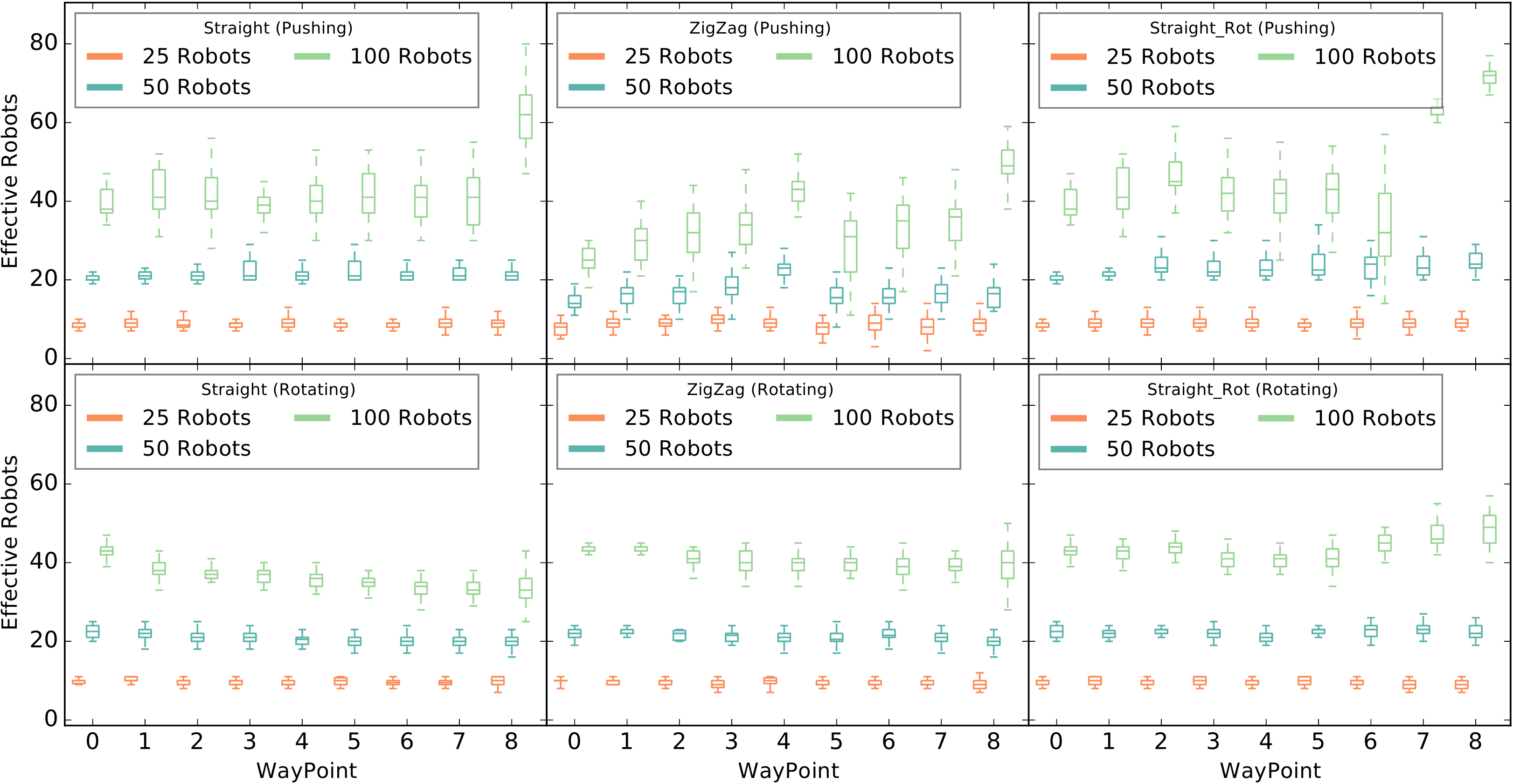}
\caption{Number of robots that were effective in pushing and rotating at different waypoints of the three benchmark paths.}
\label{fig:perf}
\end{figure}%
Fig.~\ref{fig:error} shows the centroid estimation error, position
error and orientation error on each row from left to right. The
centroid estimation error increases as the robots progress along the
path, which can be explained by the distortion in formation as the
robots progress towards the final target. The centroid estimation
error for 100 robots is relatively large and shows some variability,
which could be largely influenced by the communication topology during
the centroid estimation process, as detailed in
sec.~\ref{sec:methodology}. The object position error computed using
the difference between the desired position and the ground truth
position appears stable around 0.1m, which is within our design
tolerance ($d_{tol}$). In the $straight_{rot}$ case with 100 robots,
the position error drastically increases around the final WPs, likely
due to the drift induced by the second rotation. The orientation error
accumulates slowly for the other paths, likely because the pushing
force applied towards the target induces a small torque. Without
global positioning, the error accumulates at every rotation.

Fig.~\ref{fig:perf} shows the number of effective robots for pushing
and rotating the transported object, computed using
\ref{equ:center_push}. The number of effective pushers appears to
increase slowly as the robots progress towards the final target in all
cases, which could be due to distortions of the caging formation. The
number of effective rotators stays constant for most of the cases, but
increases during rotations, due to the robots either getting closer to
the corners or the mid point of the object. This could be caused by
the large error in the estimation of the centroid resulting in a
generation of biased control input to rotate the object.

\noindent \textbf{Robot Experiments} We perform a small set of
experiments using a group of 6 KheperaIV robots. The robots use a hub
to compute and transmit the range and bearing information from a
motion capture system, for more details on the experimental setup, we
refer the reader to~\cite{Varadharajan2020}. We performed two sets
of experiments with robots transporting a foam box of size (0.285, 0.435)m: without any payload, and with
\SI{4}{\kilogram} of LiPo battery on the
object. Fig.~\ref{fig:robot_Exp} shows the trajectories followed by
the robots and the inter-robot distance during the 3 runs without any
payload and one run with the payload. It can be observed that the
robots were able to consistently reach the target (0,0.9) by following
the 3 WPs placed at every 0.3m. The inter-robot distance between the
two adjacent robots in a cage approximates well the desired
$I_d=0.45m$ at caging and it is maintained consistently during pushing
in all runs with a maximum standard deviation of 0.1m.
\begin{figure}[tbp]
\centering
	\includegraphics[width=0.45\textwidth]{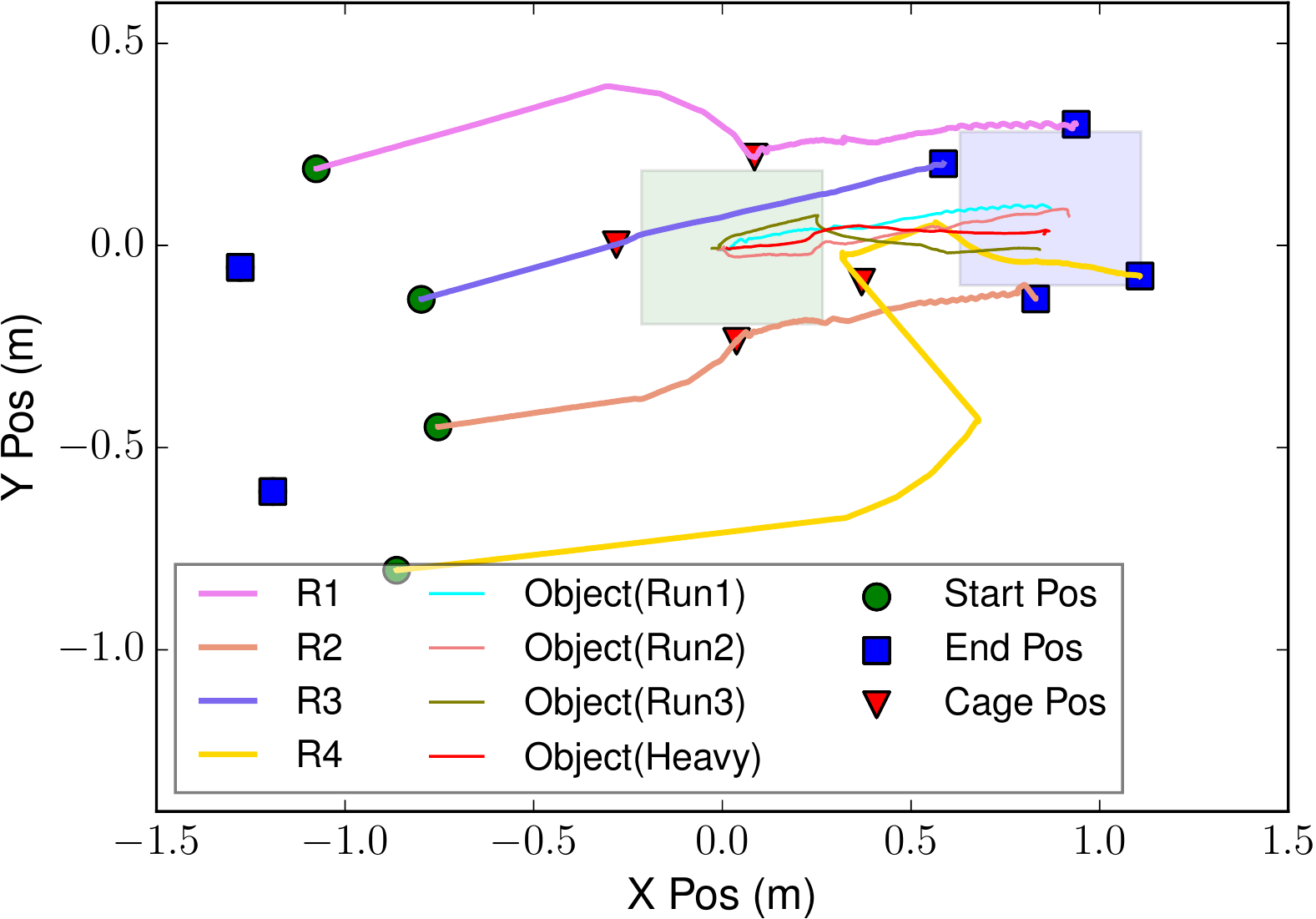}
	\includegraphics[width=0.45\textwidth]{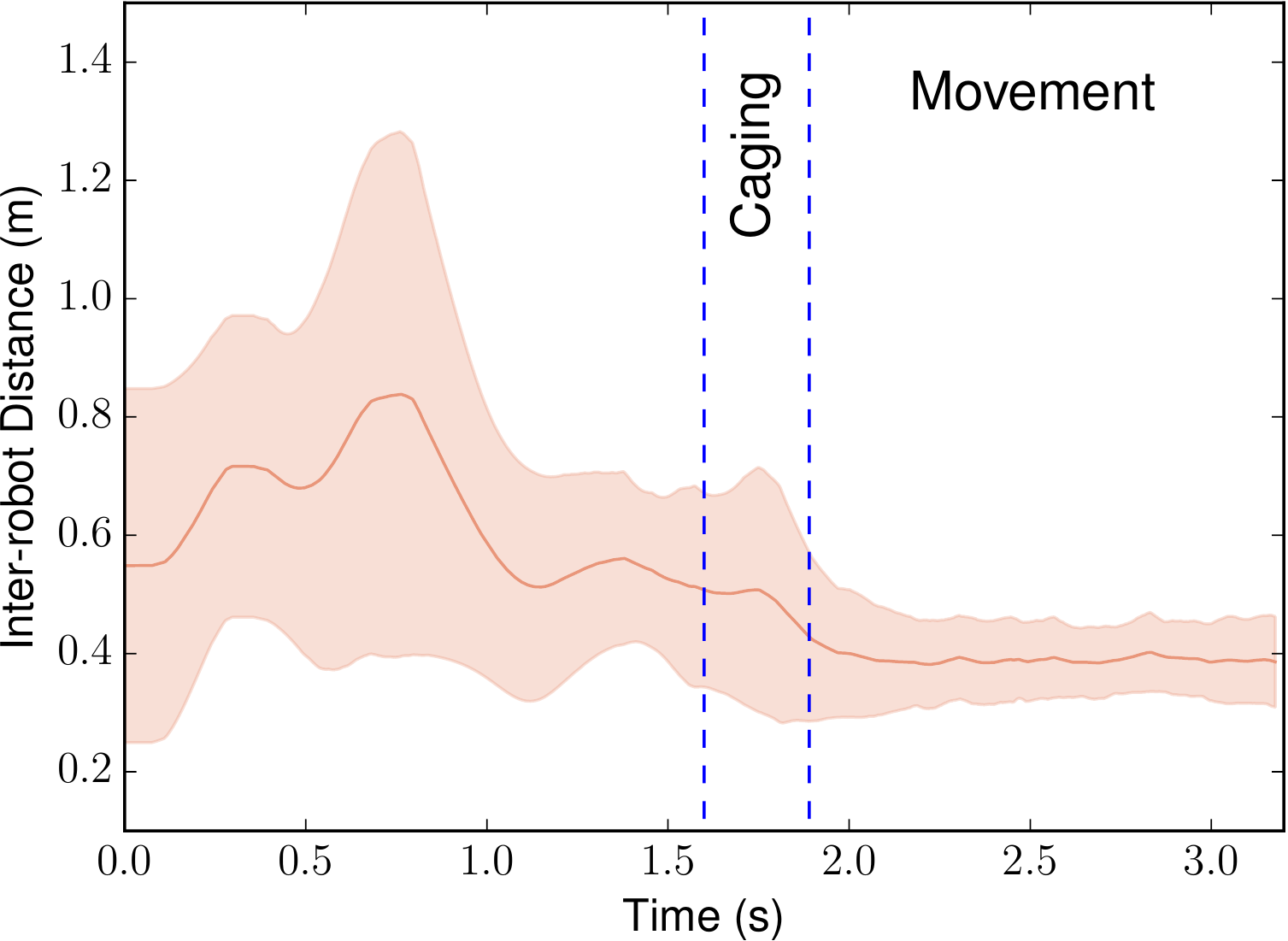}
	\caption{Trajectories taken by the robots (left) and the inter-robot distance between the two adjacent robot in the cage (right).}
\label{fig:robot_Exp}
\end{figure}%

\section{Conclusions \label{sec:conclusions}}
We propose a decentralized algorithm to cage an arbitrary-shaped
object and transport it along a desired path consisting of a set of
object poses. The robots periodically estimate the centroid of the
object based on the positional information shared by the robots caging
the object, and use this information to transport the object. We study
the performance of our algorithm using a large set of simulation
experiments with up to 100 robots traversing 3 benchmark paths and a
small set of experiments on KheperaIV robots. As future work, we plan
to implement a path planner to provide the object path in a cluttered
environment.

%
%
%
\bibliographystyle{plain}
\bibliography{dars}

\begin{thebibliography}{10}

\bibitem{alkilabi2017cooperative}
Muhanad H~Mohammed Alkilabi, Aparajit Narayan, and Elio Tuci.
\newblock Cooperative object transport with a swarm of e-puck robots:
  robustness and scalability of evolved collective strategies.
\newblock {\em Swarm intelligence}, 11(3-4):185--209, 2017.

\bibitem{melo2016collaborative}
Ramon S~Melo B, Douglas~G Macharet, Mario Fernandom~M Campos, Ramon~S Melo,
  Douglas~G Macharet, and Mario Fernandom~M Campos.
\newblock {Collaborative object transportation using heterogeneous robots}.
\newblock In {\em Robotics}, pages 172--191. Springer, 2016.

\bibitem{Chen2015}
Jianing Chen, Melvin Gauci, Wei Li, Andreas Kolling, and Roderich Gro{\ss}.
\newblock {Occlusion-Based Cooperative Transport with a Swarm of Miniature
  Mobile Robots}.
\newblock {\em IEEE Transactions on Robotics}, 31(2):307--321, 2015.

\bibitem{Choi2009}
H.~{Choi}, L.~{Brunet}, and J.~P. {How}.
\newblock Consensus-based decentralized auctions for robust task allocation.
\newblock {\em IEEE Transactions on Robotics}, 25(4):912--926, 2009.

\bibitem{cotsakis2019decentralized}
Ryan Cotsakis, David St-Onge, and Giovanni Beltrame.
\newblock Decentralized collaborative transport of fabrics using micro-uavs.
\newblock In {\em 2019 International Conference on Robotics and Automation
  (ICRA)}, pages 7734--7740. IEEE, 2019.

\bibitem{dai2016symmetric}
Yanyan Dai, YoonGu Kim, SungGil Wee, DongHa Lee, and SukGyu Lee.
\newblock Symmetric caging formation for convex polygonal object transportation
  by multiple mobile robots based on fuzzy sliding mode control.
\newblock {\em ISA transactions}, 60:321--332, 2016.

\bibitem{fink2007composition}
Jonathan Fink, Nathan Michael, and Vijay Kumar.
\newblock Composition of vector fields for multi-robot manipulation via caging.
\newblock In {\em Robotics: Science and Systems}, volume~3, 2007.

\bibitem{Franchi2019}
Antonio Franchi, Antonio Petitti, and Alessandro Rizzo.
\newblock {Distributed Estimation of State and Parameters in Multiagent
  Cooperative Load Manipulation}.
\newblock {\em IEEE Transactions on Control of Network Systems}, 6(2):690--701,
  2019.

\bibitem{Gabellieri2020}
Chiara Gabellieri, Marco Tognon, Dario Sanalitro, Lucia Pallottino, and Antonio
  Franchi.
\newblock {A study on force-based collaboration in swarms}.
\newblock {\em Swarm Intelligence}, 14(1):57--82, 2020.

\bibitem{Gelblum2015}
Aviram Gelblum, Itai Pinkoviezky, Ehud Fonio, Abhijit Ghosh, Nir Gov, and Ofer
  Feinerman.
\newblock {Ant groups optimally amplify the effect of transiently informed
  individuals}.
\newblock {\em Nature Communications}, 6, 2015.

\bibitem{gerkey2002pusher}
Brian~P Gerkey and Maja~J Mataric.
\newblock {Pusher-watcher: An approach to fault-tolerant tightly-coupled robot
  coordination}.
\newblock In {\em International Conference on Robotics and Automation (ICRA)},
  volume~1, pages 464--469. IEEE, 2002.

\bibitem{mesbahi2010graph}
Mehran Mesbahi and Magnus Egerstedt.
\newblock {\em Graph theoretic methods in multiagent networks}, volume~33.
\newblock Princeton University Press, 2010.

\bibitem{Pereira2004}
Guilherme~A.S. Pereira, Mario~F.M. Campos, and Vijay Kumar.
\newblock {Decentralized algorithms for multi-robot manipulation via caging}.
\newblock {\em International Journal of Robotics Research}, 23(7-8):783--795,
  2004.

\bibitem{Petersen2019}
Kirstin~H. Petersen, Nils Napp, Robert Stuart-Smith, Daniela Rus, and Mirko
  Kovac.
\newblock {A review of collective robotic construction}.
\newblock {\em Science Robotics}, 4(28), 2019.

\bibitem{PinciroliBuzz2016}
Carlo Pinciroli and Giovanni Beltrame.
\newblock Buzz: An extensible programming language for heterogeneous swarm
  robotics.
\newblock In {\em International Conference on Intelligent Robots and Systems},
  pages 3794--3800, October 2016.

\bibitem{Pinciroli:vs:2016}
Carlo Pinciroli, Adam Lee-Brown, and Giovanni Beltrame.
\newblock A tuple space for data sharing in robot swarms.
\newblock In {\em Proceedings of the 9th EAI International Conference on
  Bio-inspired Information and Communications Technologies}, BICT'15, pages
  287--294, 2016.

\bibitem{Pinciroli:SI2012}
Carlo Pinciroli, Vito Trianni, Rehan O'Grady, Giovanni Pini, Arne Brutschy,
  Manuele Brambilla, Nithin Mathews, Eliseo Ferrante, Gianni {Di Caro},
  Frederick Ducatelle, Mauro Birattari, Luca~Maria Gambardella, and Marco
  Dorigo.
\newblock {ARGoS}: a modular, parallel, multi-engine simulator for multi-robot
  systems.
\newblock {\em Swarm Intelligence}, 6(4):271--295, 2012.

\bibitem{Rimon1999}
Elon Rimon and Andrew Blake.
\newblock {Caging Planar Bodies by One-Parameter Two-Fingered Gripping
  Systems}.
\newblock {\em The International Journal of Robotics Research}, 18(3):299--318,
  1999.

\bibitem{rosenfeld2016human}
Ariel Rosenfeld, A~Noa, O~Maksimov, and S~Kraus.
\newblock Human-multi-robot team collaboration for efficent warehouse
  operation.
\newblock {\em Autonomous Robots and Multirobot Systems (ARMS)}, 2016.

\bibitem{Rubenstein2013}
Michael Rubenstein, Adrian Cabrera, Justin Werfel, Golnaz Habibi, James
  McLurkin, and Radhika Nagpal.
\newblock {Collective transport of complex objects by simple robots}.
\newblock {\em Proceedings of the 2013 International Conference on Autonomous
  Agents and Multi-agent Systems}, pages 47--54, 2013.

\bibitem{Rubenstein795}
Michael Rubenstein, Alejandro Cornejo, and Radhika Nagpal.
\newblock {Programmable self-assembly in a thousand-robot swarm}.
\newblock {\em Science}, 345(6198):795--799, 2014.

\bibitem{Varadharajan2020}
V.~S. {Varadharajan}, D.~{St-Onge}, B.~{Adams}, and G.~{Beltrame}.
\newblock Swarm relays: Distributed self-healing ground-and-air connectivity
  chains.
\newblock {\em IEEE Robotics and Automation Letters}, 5(4):5347--5354, 2020.

\bibitem{wan2012}
Weiwei Wan, Rui Fukui, Masamichi Shimosaka, Tomomasa Sato, and Yasuo Kuniyoshi.
\newblock {Cooperative manipulation with least number of robots via robust
  caging}.
\newblock {\em IEEE/ASME International Conference on Advanced Intelligent
  Mechatronics, AIM}, pages 896--903, 2012.

\bibitem{Wan2020}
Weiwei Wan, Boxin Shi, Zijian Wang, and Rui Fukui.
\newblock {Multirobot Object Transport via Robust Caging}.
\newblock {\em IEEE Transactions on Systems, Man, and Cybernetics: Systems},
  50(1):270--280, 2020.

\bibitem{wang2006cooperative}
Ying Wang and Clarence~W de~Silva.
\newblock Cooperative transportation by multiple robots with machine learning.
\newblock In {\em 2006 IEEE International Conference on Evolutionary
  Computation}, pages 3050--3056. IEEE, 2006.

\bibitem{wang2004control}
ZhiDong Wang, Yasuhisa Hirata, and Kazuhiro Kosuge.
\newblock Control a rigid caging formation for cooperative object
  transportation by multiple mobile robots.
\newblock In {\em IEEE International Conference on Robotics and Automation,
  2004. Proceedings. ICRA'04. 2004}, volume~2, pages 1580--1585. IEEE, 2004.

\bibitem{Wang2018}
Zijian Wang, Sumeet Singh, Marco Pavone, and Mac Schwager.
\newblock {Cooperative Object Transport in 3D with Multiple Quadrotors Using No
  Peer Communication}.
\newblock {\em Proceedings - IEEE International Conference on Robotics and
  Automation}, pages 1064--1071, 2018.

\end{thebibliography}

\end{document}